\documentclass{llncs}
\usepackage{makeidx}  
\usepackage{bm}
\usepackage{amsmath}
\usepackage{amssymb}
\usepackage{algorithmicx}
\usepackage{algorithm}
\usepackage{algpseudocode}
\usepackage{authblk}
\usepackage{graphics}
\usepackage[dvips]{graphicx}
\usepackage{subfigure}
\usepackage{epstopdf}
\newcommand{\argmin}{\arg\!\min}

\begin{document}
\title{Unsupervised Feature Analysis with \\ Class Margin Optimization}
%
%
\author{
Sen Wang\inst{1}\thanks{sen.wang@uq.edu.au}\and
Feiping Nie\inst{2}\thanks{feiping.nie@gmail.com} \and
Xiaojun Chang\inst{3}\thanks{cxj273@gmail.com} \\
Lina Yao\inst{4}\thanks{lina@cs.adelaide.edu.au}\and
Xue Li\inst{1}\thanks{xueli@itee.uq.edu.au} \and
Quan Z. Sheng\inst{4}\thanks{michael.sheng@adelaide.edu.au}
}
\institute
{
School of ITEE, The University of Queensland, Australia. \and
Center for OPTIMAL, Northwestern Polytechnical University, Shaanxi, China. \and
Center for QCIS, University of Technology Sydney, Australia. \and
School of CS, The University of Adelaide, Australia.
}
%
%
%

\maketitle              

\begin{abstract}
Unsupervised feature selection has been always attracting research attention in the communities of machine learning and data mining for decades. In this paper, we propose an unsupervised feature selection method seeking a feature coefficient matrix to select the most distinctive features. Specifically, our proposed algorithm integrates the Maximum Margin Criterion with a sparsity-based model into a joint framework, where the class margin and feature correlation are taken into account at the same time. To maximize the total data separability while preserving minimized within-class scatter simultaneously, we propose to embed K-means into the framework generating pseudo class label information in a scenario of unsupervised feature selection. Meanwhile, a sparsity-based model, $\ell_{2,p}$-norm, is imposed to the regularization term to effectively discover the sparse structures of the feature coefficient matrix. In this way, noisy and irrelevant features are removed by ruling out those features whose corresponding coefficients are zeros. To alleviate the local optimum problem that is caused by random initializations of K-means, a convergence guaranteed algorithm with an updating strategy for the clustering indicator matrix, is proposed to iteratively chase the optimal solution. Performance evaluation is extensively conducted over six benchmark data sets. From plenty of experimental results, it is demonstrated that our method has superior performance against all other compared approaches.     

\end{abstract}
%

\section{Introduction}
Over the past few years, data are more than often represented by high-dimensional features in a number of research fields, such as data mining \cite{wu2014data}, computer vision \cite{sun2014deep}, etc. With the inventions of such many sophisticated data representations, a problem has been never lack of research attention: How to select the most distinctive features from high-dimensional data for subsequent learning tasks, e.g. classification? To answer this question, we take two points into account. First, the number of selected features should be smaller than the one of all features. Due to a lower dimensional representation, the subsequent learning tasks with no doubt can gain benefit in terms of efficiency \cite{ZhuHYSXL13}. Second, the selected features should have more discriminant power than the original all features. Many previous works have proven that removing those noisy and irrelevant features can improve discriminant power in most cases. In light of advantages of feature selection, different new algorithms have been flourished with various types of applications recently.

According to the types of supervision, feature selection can be generally divided into three categories, i.e. supervised, semi-supervised, and unsupervised feature selection algorithms. Representative supervised feature selection algorithms include Fisher score \cite{duda2012pattern}, Relief\cite{kira1992practical} and its extension, ReliefF \cite{kononenko1994estimating}, information gain \cite{raileanu2004theoretical}, etc \cite{senchang2014,YangMHS13}. Label information of training data points is utilized to guide the supervised feature selection methods to seek distinctive subsets of features with different search strategies, i.e. \textit{complete search, heuristic search, and non-deterministic search}. In the real world, class information is quite limited, resulting in the development of semi-supervised feature selection methods \cite{xu2010discriminative,han2014semisupervised,chang2014convex,chang2014semi}, in which both labeled and unlabeled data are utilized. 

In unsupervised scenarios, feature selection is more challenging, since there is no class information to use for selecting features. In the literature, unsupervised feature selection can be roughly categorized into three groups, i.e. \textit{filter}, \textit{wrapper}, and \textit{embedded methods}. Filter-based unsupervised feature selection methods rank features according to some intrinsic properties of data. Then those features with higher scores are selected for the further learning tasks. The selection is independent to the consequent process. For example, He et al. \cite{he2005laplacian} assume that data from the same class are often close to each other and use the locality preserving power of data, also termed as \textit{Laplacian Score}, to evaluate importance degrees of features. In \cite{zhao2007spectral}, a unified framework has been proposed for both supervised and unsupervised feature selection schemes using a spectral graph. Tabakhi et al. \cite{tabakhi2014unsupervised} have proposed an unsupervised feature selection method to select the optimal feature subset in an iterative algorithm, which is based on ant colony optimization. 
Wrapper-based methods as a more sophisticated way wrap learning algorithms to yield learned results that will be used to select distinctive subsets of features. In \cite{maugis2009variable}, for instance, the authors have developed a model that selects relevant features using two backward stepwise selection algorithms without prior knowledges of features. Normally, wrapper-based methods have better performance than filter-based methods, since they use learning algorithms. Unfortunately, the disadvantage is that the computation of wrapper methods is more expensive. Embedded methods are seeking a trade-off between them by integrating feature selection and clustering together into a joint framework. Because clustering algorithms are able to provide pseudo labels which can reflect the intrinsic information of data, some works \cite{cai2010unsupervised,li2012unsupervised,qian2013robust,wang2015embedded} incorporate different clustering algorithms in objective functions to select features.

Most of the existing unsupervised feature selection methods \cite{he2005laplacian,zhao2007spectral,qian2013robust,WangNH14,li2012unsupervised,HouNLYW14} rely on a graph, e.g. \textit{graph Laplacian}, to reflect intrinsic relationships among data, labeled and unlabeled. When the number of data is extremely large, the computational burden of constructing a graph Laplacian is significantly heavy. Meanwhile, some traditional feature selection algorithms \cite{he2005laplacian,duda2012pattern} neglect correlations among features. The distinctive features are individually selected according to the importance of each feature rather than taking correlations among features into account. Recently, exploiting feature correlations has attracted much research attention \cite{yang2011l2,qian2013robust,chang2015ijcai,du2014multiple,NieHCD10,YangZWP08}. It has proven that discovering feature correlation is beneficial to feature selection. 

In this paper, we propose a graph-free method to select features by combining Maximum Margin Criterion with feature correlation mining into a joint framework. Specifically, the method, on one hand, aims to learn a feature coefficient matrix which linearly combines features to maximize the class margins. With the increase of the separability of the entire transformed data by maximizing the total scatter, the proposed method also expects distances between data points within the same class to be minimized after the linear transformation by the coefficient matrix. Since there is no class information can be borrowed from, K-means clustering is jointly embedded in the framework to provide pseudo labels. Inspired by recent feature selection works using sparsity-based model on the regularization term \cite{chang2014semi}, on the other hand, the proposed algorithm learns sparse structural information of the coefficient matrix, with the goal of reducing noisy and irrelevant features by removing those features whose coefficients are zeros. The main contributions of this paper can be summarized as follows:
\begin{itemize}
\item The proposed method makes efforts to maximize class margins in a framework, where simultaneously considers the separability of the transformed data and distances between the transformed data within the same class. Besides, a sparsity-based regularization model is jointly applied on the feature coefficient matrix to analyze correlations among features in an iterative algorithm;
\item K-means clustering is embedded into the framework generating cluster labels, which can be used as pseudo labels. Both maximizing class margins and learning sparse structures can benefit from generated pseudo labels during each iteration; 
\item Because the performance of K-means is dominated by the initialization, we propose a strategy to avoid our algorithm rapidly converge to a local optimum, which is largely ignored by most of existing approaches using K-means clustering. Theoretical proof of convergence is also provided.
\item We have conducted extensive experiments over six benchmark datasets. The experimental results show that our method has better performance than all the compared unsupervised algorithms.
\end{itemize}

The rest of this paper is organized as follows: Notations and definitions that are used throughout the entire paper will be given in section 2. Our method will be elaborated in section 3, followed by proposing its optimization with an algorithm to guarantee the convergence property in section 4. In section 5, extensive experimental results are reported with related analysis. Lastly, the conclusion of this paper will be given in section 6.
\section{Notations and Definitions}
To give a better understanding of the proposed method, notations and definitions which are used throughout this paper are summarized in this section. Matrices and vectors are written as boldface uppercase letters and boldface lowercase letters, respectively. Given a data set denoted as $\bm{X}=[\bm{x}_1, \ldots, \bm{x_n}] \in \mathbb{R}^{d\times n}$, where $n$ is the number of training data and $d$ is the feature dimension. The mean of data is denoted as $\bm{\bar{x}}$. The feature coefficient matrix, $\bm{W} \in \mathbb{R}^{d\times d^\prime}$, linearly combines data features as $\bm{W}^T\bm{X}$, $d^\prime$ is the feature dimension after the linear transformation. Given a cluster centroid matrix for the transformed data, $\bm{G} = [\bm{g}_1, \ldots, \bm{g}_c]\in\mathbb{R}^{d^\prime \times c}$, its cluster indicator of transformed $\bm{x}_i$ is represented as $\bm{u}_i = [u_{i1}, \ldots, u_{ic}]$. $c$ is the number of centroids. If transformed $\bm{x}_i$ belongs to the $j$-th cluster, $u_{ij}=1$, $j=1,\ldots, c$. Otherwise, $u_{ij}=0$. Correspondingly, the cluster indicator matrix is $\bm{U}=[\bm{u}_1^T, \ldots, \bm{u}_n^T]^T\in \mathbb{R}^{n\times c}$.

For an arbitrary matrix $\bm{M} \in \mathbb{R}^{r\times l}$, its $\ell_{2,p}$-norm is defined as:
\begin{equation}
\label{l2pnorm}
	\| \bm{M} \|_{2,p} = \left[\sum\limits_{i=1}^r \left(\sum\limits_{j=1}^l M_{ij}^2\right)^{\frac{p}{2}}\right]^{\frac{1}{p}}
\end{equation}
The $i$-th row of $\bm{M}$ is represented by $\bm{M}^i$. 
The between-class, within-class and total scatter matrices of data are respectively defined as: 
\begin{equation}
	\begin{aligned}
\label{eq_sb}
\bm{S_b} & = \sum\limits_{i=1}^c n_i(\bm{\bar{x}_i} - \bm{\bar{x}})(\bm{\bar{x}_i} - \bm{\bar{x}})^T,         \\
\bm{S_w} & = \sum\limits_{i=1}^c \sum\limits_{j=1}^{n_i}(\bm{x_j} - \bm{\bar{x}}_i)(\bm{x_j} - \bm{\bar{x}}_i)^T, \\
\bm{S_t} & = \sum\limits_{i=1}^n (\bm{x_i} - \bm{\bar{x}})(\bm{x_i} - \bm{\bar{x}})^T                        
	\end{aligned}
\end{equation}
where $n_i$ is the number of data for the $c$-th class. $\bm{S_t}=\bm{S_w}+\bm{S_b}$. 
Other notations and definitions will be explained when they are in use.
\section{Proposed Method}
We now introduce our proposed method for unsupervised feature selection. To exploit distinctive features, an intuitive way is to find a linear transformation matrix which can project the data into a new space where the original data are more separable. PCA is the most popular approach to analyze the separability of features. PCA aims to seek directions on which transformed data have max variances. In other words, PCA is to maximize the separability of linearly transformed data by maximizing the covariance: $\max\limits_{\bm{W}} \sum\limits_{i=1}^n (\bm{W}^T(\bm{x}_i - \bm{\bar{x}}))^T(\bm{W}^T(\bm{x}_i - \bm{\bar{x}}))$. Without losing the generality, we assume the data has zero mean, i.e. $\bar{\bm{x}}=0$. Recall the definition of total scatter of data, PCA is equivalent to maximize the total scatter of data. However, if only total scatter is considered as a separability measure, the within-class scatter might be also geometrically maximized with the maximization of the total scatter. This is not helpful to distinctive feature discovery. The representative model, LDA, solves this problem by maximizing Fisher criterion: $\max\limits_{\bm{W}}\frac{\bm{W}^T\bm{S_bW}}{\bm{W}^T\bm{S_wW}}$. However, LDA and its variants require class information to construct between-class and within-class scatter matrices \cite{DBLP:journals/tnn/CHANGXJ15}, which is not suitable for unsupervised feature selection. Before we give the objective that can solve the aforementioned problem, we first look at a supervised feature selection framework:
\begin{equation}
	\begin{aligned}
	\label{eq_framework_initial}
	\max\limits_{\bm{W}} &\sum\limits_{i=1}^n (\bm{W}^T \bm{x}_i)^T(\bm{W}^T \bm{x}_i) - \alpha\sum\limits_{i=1}^c \sum\limits_{j=1}^{n_i}(\bm{W}^T(\bm{x_j} - \bm{\bar{x}}_i))^T(\bm{W}^T(\bm{x_j} - \bm{\bar{x}}_i)) - \beta\Omega(\bm{W}) \\
\textbf{s.t.}&\quad \bm{W}^T\bm{W}=\bm{I}, 
	\end{aligned}
\end{equation}
where $\alpha$ and $\beta$ are regularization parameters. In this framework, the first term is to maximize the total scatter, while the second term is to minimize the within-class scatter. The third part is a sparsity-based regularization term which controls the sparsity of $\bm{W}$. This model is quite similar with the classical LDA-based methods. Due to there is no class information in the unsupervised scenario, we need virtual labels to minimize the distances between data within the same class while maximize the total separability at the same time. To achieve this goal, we apply K-means clustering in our framework to replace the ground truth by generating cluster indicators of data. Given $c$ centroids $\bm{G}=[\bm{g}_1,\ldots,\bm{g}_c] \in\mathbb{R}^{d^\prime \times c}$, the objective function of the traditional K-means algorithm aims to minimize the following function:
\begin{equation}
	\label{eq_kmeans}
	\begin{aligned}
		&\sum\limits_{i=1}^c \sum\limits_{\bm{y}_j \in \mathcal{Y}_i} (\bm{y}_j - \bm{g}_i)^T(\bm{y}_j - \bm{g}_i) \\
		=&\sum\limits_{i=1}^n (\bm{y}_i - \bm{G}\bm{u}_i^T)^T(\bm{y}_i - \bm{G}\bm{u}_i^T),\\
\end{aligned}
\end{equation}
where $\bm{y}_i = \bm{W}^T\bm{x}_i$. Note that K-means is used to assign cluster labels, which are used as pseudo labels, to minimize the within-class scatter after the linear transformation by $\bm{W}$. Then, we can substitute \eqref{eq_kmeans} into \eqref{eq_framework_initial}:
\begin{equation}
	\label{eq_framework_kmeans}
	\begin{aligned}
		\max\limits_{\bm{W}} &\sum\limits_{i=1}^n (\bm{W}^T\bm{x}_i)^T(\bm{W}^T\bm{x}_i) - \alpha \sum\limits_{i=1}^n (\bm{W}^T\bm{x}_i - \bm{G}\bm{u}_i^T)^T(\bm{W}^T\bm{x}_i - \bm{G}\bm{u}_i^T)- \beta \Omega(\bm{W})\\
\textbf{s.t.}&\quad \bm{W}^T\bm{W}=\bm{I}, 
	\end{aligned}
\end{equation}
As mentioned above, the sparsity-based regularization term has been widely used to find out correlated structures among features. The motivation behind this is to exploit sparse structures of the feature coefficient matrix. By imposing the sparse constraint, some of the rows of the feature coefficient matrix shrink to zeros. Those features corresponding to non-zero coefficients are selected as the distinctive subset of features. In this way, noisy and redundant features can be removed. This sparsity-based regularization has been applied in various problems. Inspired by the \textit{"shrinking to zero"} idea, we utilize a sparsity model to uncover the common structures shared by features. To achieve that goal, we propose to minimize the $\ell_{2,p}$-norm of the coefficient matrix, $\| \bm{W} \|_{2,p}, (0<p<2)$. From the definition of $\| \bm{W} \|_{2,p}$ in \eqref{l2pnorm}, outliers or negative impact of the irrelevant $\bm{w^i}$'s are suppressed by minimizing the $\ell_{2,p}$-norm. Note that $p$ is a parameter that controls the degree of correlated structures among features. The lower $p$ is, the more shared structures among are expected to exploit. After a number of optimization steps, the optimal feature coefficient matrix, $\bm{W}$, can be obtained. Thus, we impose the $\ell_{2,p}$-norm on the regularization term and re-write the objective function in a matrix representation as follows:
\begin{equation}
\label{eq_obj_final}
\begin{aligned}
\max_{\bm{W}, \bm{G}, \bm{U}} &Tr(\bm{W}^T\bm{S_t}\bm{W}) - \alpha \|\bm{W}^T\bm{X}-\bm{G}\bm{U^T}\|_F^2 - \beta \|\bm{W}\|_{2,p}\\
\textbf{s.t.}&\quad \bm{W}^T\bm{W}=\bm{I}, 
\end{aligned}
\end{equation}
where $\bm{U}$ is an indicator matrix. $Tr(\cdot)$ is trace operator, while $\|\cdot\|_F^2$ is the Frobenius norm of a matrix. Our proposed method integrates the Maximum Margin Criterion and sparse regularization into a joint framework. Embedding K-means into the framework not only minimizes the distances between within-class data while maximizing total data separability, but also provides cluster labels. The cluster centroids generated by K-means can further guide the sparse structure learning on the feature coefficient matrix in each iterative step of our solution, which will be explained in the next section. We name this method for the unsupervised feature analysis with class margin optimization as \textbf{UFCM}. 

\section{Optimization}
In this section, we present our solution to the objective function in \eqref{eq_obj_final}. Since the $\ell_{2,p}$-norm is used to exploit sparse structures, the objective function cannot be solved in a closed form. Meanwhile, the objective function is not jointly convex with respect to three variables, i.e. $\bm{W,G,U}$. Thus, we propose to solve the problem as follows.

We define a diagonal matrix $\bm{D}$ whose diagonal entries are defined as: 
\begin{equation}
\label{eq_2pnorm_dmatrix}
\bm{D}^{ii} = \frac{1}{\frac{2}{p}\| \bm{w}^i\|_2^{2-p}}.
\end{equation}
The objective function in \eqref{eq_obj_final} is equivalent to:
\begin{equation}
\label{eq_obj_equivalent}
\begin{aligned}
\max_{\bm{W}, \bm{G}, \bm{U}} &Tr(\bm{W}^T\bm{S_t}\bm{W}) - \alpha \|\bm{W}^T\bm{X}-\bm{G}\bm{U^T}\|_F^2 - \beta Tr(\bm{W}^T \bm{DW})\\
\textbf{s.t.}&\quad \bm{W}^T\bm{W}=\bm{I}
\end{aligned}
\end{equation}
We propose to optimize the objective function in two steps in each iteration as follows:

(1) Fix $\bm{W,G}$ and optimize $\bm{U}$:

When $\bm{W}$ is fixed, the first and third terms can be viewed as constants. While the second term can be viewed as the objective function of K-means, assigning cluster labels to each data. Also, the cluster centroid matrix $\bm{G} = [\bm{g}_1, \ldots, \bm{g}_c]$ is also fixed, the optimal $\bm{U}$ is:
\begin{equation}
U_{ij} = \left\{ \quad\begin{array}{cl}
1, & \qquad j=\argmin\limits_{k} \|\bm{W}^T\bm{x}_i  - \bm{g}_k \|_F^2, \\
0, & \qquad \text{Otherwise.} 
\end{array} \right.
\end{equation}
This is equivalent to perform K-means on the transformed data, $\bm{W}^T\bm{X}$, which means the solution is unique.

(2) Fix $\bm{U}$ and optimize $\bm{W, G}$:

After fixing the indicator matrix, $\bm{U}$, we set the derivative of Equation \eqref{eq_obj_equivalent} with respect to $\bm{G}$ equal to 0:
\begin{equation}
\label{eq_G}
\begin{aligned}
-\alpha \frac{\partial{Tr(\bm{W}^T\bm{X} - \bm{GU}^T)^T(\bm{W}^T\bm{X} - \bm{GU}^T)}}{\partial \bm{G}}& = 0 \\
\Rightarrow -2\alpha\bm{W}^T\bm{XU} + 2\alpha \bm{GU}^T\bm{U} &= 0 \\ 
\Rightarrow \bm{G} = \bm{W}^T\bm{XU}(\bm{U}^T\bm{U})^{-1}
\end{aligned}
\end{equation}
Substituting Equation \eqref{eq_G} into Equation \eqref{eq_obj_equivalent}, we have:
\begin{equation}
\label{eq_obj_equivalent_1}
\begin{aligned}
Tr&(\bm{W}^T\bm{S_t}\bm{W}) - \alpha \|\bm{W}^T\bm{X}-\bm{W}^T\bm{XU}(\bm{U}^T\bm{U})^{-1}\bm{U^T}\|_F^2 - \beta Tr(\bm{W}^T \bm{DW})\\
= \alpha Tr&\left((\bm{W}^T\bm{XU}(\bm{U}^T\bm{U})^{-1}\bm{U^T}-\bm{W}^T\bm{X})(\bm{W}^T\bm{X}-\bm{W}^T\bm{XU}(\bm{U}^T\bm{U})^{-1}\bm{U^T})^T \right)\\
+ Tr&(\bm{W}^T\bm{S_t}\bm{W}) - \beta Tr(\bm{W}^T \bm{DW}) \\
= \alpha Tr&\left(\bm{W}^T\bm{X}\bm{U}(\bm{U}^T\bm{U})^{-1}\bm{U}^T\bm{X}^T\bm{W}-\bm{W}^T\bm{XX}^T\bm{W}\right) \\
+ Tr&(\bm{W}^T\bm{S_t}\bm{W}) - \beta Tr(\bm{W}^T \bm{DW}) \\
= Tr&[\bm{W}^T(\bm{S_t} +\alpha\bm{X}\bm{U}(\bm{U}^T\bm{U})^{-1}\bm{U}^T\bm{X}^T - \alpha\bm{XX}^T-\beta D)\bm{W} ]
\end{aligned}
\end{equation}
Thus, the objective function becomes:
\begin{equation}
\label{eq_obj_equivalent_2}
\begin{aligned}
\max\limits_{\bm{W}} Tr&[\bm{W}^T(\bm{S_t} +\alpha\bm{X}\bm{U}(\bm{U}^T\bm{U})^{-1}\bm{U}^T\bm{X}^T - \alpha\bm{XX}^T-\beta D)\bm{W} ] \\
\textbf{s.t.}&\quad \bm{W}^T\bm{W}=\bm{I}
\end{aligned}
\end{equation}
The objective function can be then solved by performing eigen-decomposition of the following formula:
\begin{equation}
\label{eigndecomp}
\bm{S_t} +\alpha\bm{X}\bm{U}(\bm{U}^T\bm{U})^{-1}\bm{U}^T\bm{X}^T - \alpha\bm{XX}^T-\beta D
\end{equation} 
The optimal $\bm{W}$ can be determined by choosing $d^\prime$ eigenvectors corresponding to $d^\prime$ largest eigenvalues, $d^\prime \leq d$. Our proposed method can be solved by above steps in an iterative algorithm. Each step can obtain the corresponding optimum. As the cluster indicator matrix $\bm{U}$ is initialized by K-means, the performance of our algorithm is determined by the initialization of K-means. To alleviate the local optimum problem, an update strategy for $\bm{U}$ is demanded. Generally speaking, we randomly initialize $\bm{U}$ a number of times and make comparisons according to the second term in Equation \eqref{eq_obj_final}. Then we choose how to update the indicator matrix. Specifically, the optimal $\bm{U}_i^*$ and $\bm{W}_i^*$ has been derived in the $i$-th iteration. In the $(i+1)$-th iteration, we first randomly initialize $\bm{U}$ $r$ times ($r=10$ in our experiment) and combine the derived $\bm{U}_i^*$ in the $i$-th iteration as an updating candidate set: $\tilde{\bm{U}}_{i+1} = [\bm{U}_{i+1}^0, \bm{U}_{i+1}^1, \ldots, \bm{U}_{i+1}^r]$, $\bm{U}_{i+1}^0 = \bm{U}_i^*$. According to $\|\bm{W}^T\bm{X}-\bm{G}{U}^T\|_F^2$, the candidate, which yields the smallest value, is chosen to update $\bm{U}_{i+1}^*$:  
\begin{equation}
\label{eq_update_rule}
\bm{U}_{i+1}^* = \tilde{\bm{U}}_{i+1}^j, \qquad j = \argmin_{j} \|\bm{W}^T\bm{X}-\bm{G}(\tilde{\bm{U}}_{i+1}^j)^T\|_F^2
\end{equation}
where $j$ is the index of candidate set, $j = 0,1,\ldots, r$. In this way, we compare the derived cluster indicator matrix with $r$ randomly initialized counterparts to alleviate the local optimum problem. We summarize the solution in Algorithm \ref{alg} which outputs the learned feature coefficient matrix $\bm{W}$ to select distinctive features.

\begin{algorithm}[!tb]
\caption{Unsupervised Feature Analysis with Class Margin Optimization.}
\label{alg}
\begin{algorithmic}[1]
\Require 
Data matrix $\bm{X}=[\bm{x}_1, \ldots, \bm{x_n}] \in \mathbb{R}^{d\times n}$ and parameters $\alpha$ and $\beta$.
\Ensure 
Feature coefficient matrix $\bm{W}$ and cluster indicator matrix $\bm{U}$.
\State Initialize $\bm{W}$ by PCA on $\bm{X}$;
\State Initialize $\bm{U}$ by K-means on $\bm{W}^T\bm{X}$;
\Repeat
\State Compute $\bm{D}$ according to \eqref{eq_2pnorm_dmatrix};
\State Update $\bm{U}$ according to \eqref{eq_update_rule}; 
\State Update $\bm{W}$ by eigen-decomposition of \eqref{eigndecomp}; 
\State Update $\bm{G}$ according to \eqref{eq_G};

\Until{Convergence}
\end{algorithmic}
\end{algorithm}
From Algorithm \ref{alg}, it can be seen that the most computational operation is the eigen-decomposition in Equation \eqref{eigndecomp}. The computational complexity is $O(d^3)$. If the dimensionality of the data, $d$, is very high, dimensionality reduction is desirable. To analyze the convergence of our proposed method, the following proposition and its proof are given.
\begin{proposition}
Algorithm \ref{alg} monotonically increases the objective function in Equation \eqref{eq_obj_final} until convergence.

\end{proposition}
\begin{proof}
Assuming that, in the $i$-th iteration, the transformation matrix $\bm{W}$ and cluster centroid matrix $\bm{G}$ have been derived as $\bm{W}_i$ and $\bm{G}_i$. In the $(i+1)$-th iteration step, we use $\bm{W}_i$ and $\bm{G}_i$ to update $\bm{U}_{i+1}$ according to the updating strategy in \eqref{eq_update_rule}. We can have the following inequality:
\begin{equation}
\label{eq_proof_1}
\begin{aligned}
&Tr(\bm{W}_i^T \bm{S_t} \bm{W}_i) - \alpha \| \bm{W}_i^T \bm{X} - \bm{G}_i\bm{U}_i^T\|_F^2 - \beta \|\bm{W}_i\|_{2,p}\\ 
\leq & Tr(\bm{W}_i^T \bm{S_t} \bm{W}_i) - \alpha \| \bm{W}_i^T \bm{X} - \bm{G}_i\bm{U}_{i+1}^T\|_F^2 - \beta \|\bm{W}_i\|_{2,p}
\end{aligned}
\end{equation}
Similarly, when $\bm{U}_{i+1}$ is fixed to optimize $\bm{W}$ and $\bm{G}$ in the $(i+1)$-th iteration, the following inequality can be obtained according to Equation \eqref{eq_obj_equivalent_2}:
\begin{equation}
\label{eq_proof_2}
\begin{aligned}
&Tr(\bm{W}_i^T \bm{S_t} \bm{W}_i) - \alpha \| \bm{W}_i^T \bm{X} - \bm{G}_i\bm{U}_{i+1}^T\|_F^2 - \beta \|\bm{W}_i\|_{2,p}\\
\leq & Tr(\bm{W}_{i+1}^T \bm{S_t} \bm{W}_{i+1}) - \alpha \| \bm{W}_{i+1}^T \bm{X} - \bm{G}_{i+1}\bm{U}_{i+1}^T\|_F^2 - \beta \|\bm{W}_{i+1}\|_{2,p}\\
\end{aligned}
\end{equation}
\end{proof}
After combining Equation \eqref{eq_proof_1} and \eqref{eq_proof_2} together, it indicates that the proposed algorithm will monotonically increase the objective function in each iteration. It is worth noting that the algorithm is alleviating the local optimum problem raised by random initializations of K-means, rather than completely solving it. However, our algorithm can avoid to rapidly converge to a local optimum and may converge to the global optimal solution.
\section{Experiments}
In this section, experimental results will be presented together with related analysis. We compare our method with seven approaches over six benchmark datasets. Besides, we also conduct experiments to evaluate performance variations in different aspects. They are including the impact of different selected feature numbers, the validation of feature correlation analysis, and parameter sensitivity analysis. Lastly, the convergence demonstration is shown.
\subsection{Experiment Setup}
In the experiments, we have compared our method with seven approaches as follows:
\begin{itemize}
\item \textbf{All Features}: All original variables are preserved as the baseline in the experiments.
\item \textbf{Max Variance}: Features are ranked according to the variance magnitude of each feature in a descending order. The highest ranked features are selected.
\item \textbf{Spectral Feature Selection (SPEC)} \cite{zhao2007spectral}: This method employs a unified framework to select features one by one based on spectral graph theory.
\item \textbf{Multi-Cluster Feature Selection (MCFS)} \cite{cai2010unsupervised}: This unsupervised approach selects those features who make the multi-cluster structure of the data preserved best. Features are selected using spectral regression with the $\ell_1$-norm regularization.
\item \textbf{Robust Unsupervised Feature Selection (RUFS)} \cite{qian2013robust}: RUFS jointly performs robust label learning and robust feature learning. To achieve this, robust orthogonal nonnegative matrix factorization is applied to learn labels while the $\ell_{2,1}$-norm minimization is simultaneously utilized to learn the features.
\item \textbf{Nonnegative Discriminative Feature Selection (NDFS)} \cite{li2012unsupervised}: NDFS exploits local discriminative information and feature correlations simultaneously. Besides, the manifold structure information is also considered jointly.
\item \textbf{Laplacian Score (LapScore)} \cite{he2005laplacian}: This method learns and selects distinctive features by evaluating their powers of locality preserving, which is also called Laplacian Score.  
\end{itemize}

All the parameters (if any) are tuned in the range of $\{10^{-3}, 10^{-1}, 10^1, 10^3\}$ for each algorithm mentioned above and the best results are reported. The size of the neighborhood is set to 5 for any algorithm based on spectral clustering. The number of random initializations required in the update strategy in \eqref{eq_update_rule}, is set at 10 in the experiment. To measure the performance, two metrics have been used: \textit{Clustering Accuracy (ACC)} and \textit{Normalized Mutual Information (NMI)}. 

For a data point $x_i$, its ground truth label is denoted as $p_i$ and its clustering label that is produced from a clustering algorithm, is represented as $q_i$. Then, \textit{ACC} metric over a data set with $n$ data points is defined as follows:
\begin{equation}
ACC= \frac{\sum_{i=1}^n \delta(p_i,map(q_i))}{n},
\end{equation}
where $\delta(x,y)=1$ if $x=y$ and $\delta(x,y)=0$ otherwise. $map(x)$ is the \textit{best mapping function} which permutes clustering labels to match the ground truth labels using the Kuhn-Munkres algorithm. A larger \textit{ACC} means better performance. 

According to the definition in \cite{strehl2003cluster}, \textit{NMI} is defined as:
\begin{equation}
\begin{aligned}
NMI=\frac{\sum_{l=1}^c\sum_{h=1}^c t_{l,h}log(\frac{n\times t_{l,h}}{t_l\tilde{t_h}})}{\sqrt{(\sum_{l=1}^c t_l log\frac{t_l}{n})(\sum_{h=1}^c \tilde{t_h} log \frac{\tilde{t_h}}{n})}},
\end{aligned}
\end{equation}
where $t_l$ is the number of data points in the $l$-th cluster, $1\leq l \leq c$, which is generated by a clustering algorithm. While $\tilde{t_h}$ denotes the number of data points in the $h$-th ground truth cluster. $t_{l,h}$ is the number of data points which are in the intersection of the $l$-th and $h$-th clusters. Similarly, a larger \textit{NMI} means better performance.

The performance evaluations are performed over six benchmark datasets as follows: 
\begin{itemize}
\item \textbf{COIL20} \cite{nene1996columbia}: It contains 1,440 gray-scale images of 20 objects (72 images per object) under various poses. The objects are rotated through 360 degrees and taken at the interval of 5 degrees.
\item \textbf{MNIST} \cite{lecun1998gradient}: It is a large-scale dataset of handwritten digits, which has been widely used as a test bed in data mining. The dataset contains 60,000 training images and 10,000 testing images. In this paper, we use its subclass version, MNIST-S, in which one handwritten digit image per ten images, for each class, is randomly sampled from the MNIST database. There are 6,996 handwritten images with a resolution of 28$\times$28.
\item \textbf{ORL} \cite{samaria1994parameterisation}: This data set which is used as a benchmark for face recognition, consists of 40 different subjects with 10 images each. We also resize each image to 32 $\times$ 32.
\item \textbf{UMIST}: UMIST, which is also known as the Sheffield Face Database, consists of 564 images of 20 individuals. Each individual is shown in a variety of poses from profile to frontal views.
\item \textbf{USPS} \cite{hull1994adatabase}: This dataset collects 9,298 images of handwritten digits (0-9) from envelops by the U.S. Postal Service. All images have been normalized to the same size of 16 $\times$ 16 pixels in gray scale. 
\item \textbf{YaleB} \cite{georghiades2001few}: It consists of 2,414 frontal face images of 38 subjects. Different lighting conditions have been considered in this dataset. All images are reshaped into 32 $\times$ 32 pixels. 
\end{itemize}
The pixel value in data is used as the feature. Details of data sets that are used in this paper are summarized in Table \ref{tab_datasets}.

\begin{table}[t] 
\caption{Summary of data sets.}
\centering
\begin{tabular}{|l|c|c|c|c|c|c|}
\hline
& COIL20 & MNIST & ORL & UMIST & USPS & YaleB \\
\hline
Number of data & 1,440 & 6,996 & 400 & 564 & 9,298 & 2,414 \\
\hline
Number of classes & 20 & 10 & 40 & 20 & 10 & 38 \\
\hline
Feature dimensions & 1,024 & 784 & 1,024 & 644 & 256 & 1,024\\
\hline
\end{tabular}
\label{tab_datasets}
\end{table}

\begin{table}[b] 
\caption{Performance comparison (\textit{ACC}$\pm$\textit{STD}).}
\tiny
\centering
\begin{tabular}{|l|c|c|c|c|c|c|c|c|}
\hline
& COIL20 & MNIST & ORL & UMIST & USPS &YaleB\\\hline
AllFea&$0.7051\pm0.0294$&$0.6009\pm0.0063$&$0.6675\pm0.0112$&$0.4800\pm0.0115$&$0.7139\pm0.0272$&$0.1261\pm0.0025$\\\hline
MaxVar&$0.7124\pm0.0191$&$0.6239\pm0.0100$&$0.6965\pm0.0121$&$0.4984\pm0.0141$&$0.7165\pm0.0186$&$0.1291\pm0.0042$\\\hline
SPEC&$0.7105\pm0.0116$&$0.6254\pm0.0024$&$0.6645\pm0.0065$&$0.4824\pm0.0077$&$0.7037\pm0.0315$&$0.1307\pm0.0049$\\\hline
MCFS&$0.7355\pm0.0050$&$0.6299\pm0.0037$&$0.7055\pm0.0048$&$0.5239\pm0.0038$&$0.7634\pm0.0138$&$0.1355\pm0.0043$\\\hline
RUFS&$0.7365\pm0.0024$&$0.6294\pm0.0028$&$0.6920\pm0.0033$&$0.5110\pm0.0091$&$0.7659\pm0.0076$&$0.1795\pm0.0032$\\\hline
NDFS&$0.7368\pm0.0074$&$0.6291\pm0.0016$&$0.7050\pm0.0031$&$0.5243\pm0.0028$&$0.7630\pm0.0124$&$0.1315\pm0.0034$\\\hline
LapScore&$0.7126\pm0.0249$&$0.6214\pm0.0054$&$0.7100\pm0.0117$&$0.5092\pm0.0062$&$0.7089\pm0.0324$&$0.1255\pm0.0025$\\\hline
Ours&\textbf{0.7475}$\pm$\textbf{0.0076}&$\bm{0.6392}\pm\bm{0.0056}$&$\bm{0.7210}\pm\bm{0.0052}$&$\bm{0.5343}\pm\bm{0.0062}$&$\bm{0.7813}\pm\bm{0.007}$&$\bm{0.1886}\pm\bm{0.0043}$\\\hline
\end{tabular}
\label{tab_acc}
\end{table}

\begin{table}[t] 
\caption{Performance comparison (\textit{NMI}$\pm$\textit{STD}).}
\tiny
\centering
\begin{tabular}{|l|c|c|c|c|c|c|c|c|}
\hline
& COIL20 & MNIST & ORL & UMIST & USPS &YaleB\\\hline
AllFea&$0.7884\pm0.0157$&$0.5162\pm0.0027$&$0.8265\pm0.0129$&$0.6715\pm0.0069$&$0.6305\pm0.0029$&$0.1968\pm0.0017$\\\hline
MaxVar&$0.7932\pm0.0071$&$0.5314\pm0.0063$&$0.8424\pm0.0085$&$0.6825\pm0.0063$&$0.6361\pm0.0021$&$0.2123\pm0.0040$\\\hline
SPEC&$0.7866\pm0.0061$&$0.5367\pm0.0035$&$0.8232\pm0.0021$&$0.6753\pm0.0114$&$0.6215\pm0.0073$&$0.2071\pm0.0027$\\\hline
MCFS&$0.8066\pm0.0025$&$0.5367\pm0.0003$&$0.8460\pm0.0025$&$0.7005\pm0.0053$&$0.6419\pm0.0015$&$0.2024\pm0.0033$\\\hline
RUFS&$0.8045\pm0.0025$&$0.5374\pm0.0021$&$0.8430\pm0.0044$&$0.6898\pm0.0035$&$0.6468\pm0.0027$&$0.2845\pm0.0040$\\\hline
NDFS&$0.8062\pm0.0058$&$0.5376\pm0.0004$&$0.8458\pm0.0026$&$0.6981\pm0.0054$&$0.6452\pm0.0054$&$0.2048\pm0.0041$\\\hline
LapScore&$0.7920\pm0.0101$&$0.5308\pm0.0065$&$0.8421\pm0.0006$&$0.6924\pm0.0027$&$0.6291\pm0.0047$&$0.1945\pm0.0018$\\\hline
Ours&$\bm{0.8119}\pm\bm{0.0035}$&$\bm{0.5422}\pm\bm{0.0018}$&$\bm{0.8518}\pm\bm{0.0027}$&$\bm{0.7112}\pm\bm{0.0033}$&$\bm{0.6535}\pm\bm{0.0022}$&$\bm{0.2959}\pm\bm{0.0043}$\\\hline
\end{tabular}
\label{tab_nmi}
\end{table}

\subsection{Experimental Results}
To compare the performance of our proposed algorithm with others, we repeatedly perform the test five times and report the average performance results (\textit{ACC} and \textit{NMI}) with standard deviations in Tables \ref{tab_acc} and \ref{tab_nmi}. It is observed that our proposed method consistently achieves better performance than all other compared approaches across all the data sets. Besides, it is worth noting that our method is superior to those state-of-the-art counterparts that rely on a graph Laplacian (SPEC, RUFS, NDFS, LapScore).
\begin{figure}{b}
  \centering
  \subfigure[COIL20]{
    \includegraphics[width = .30\linewidth]
    {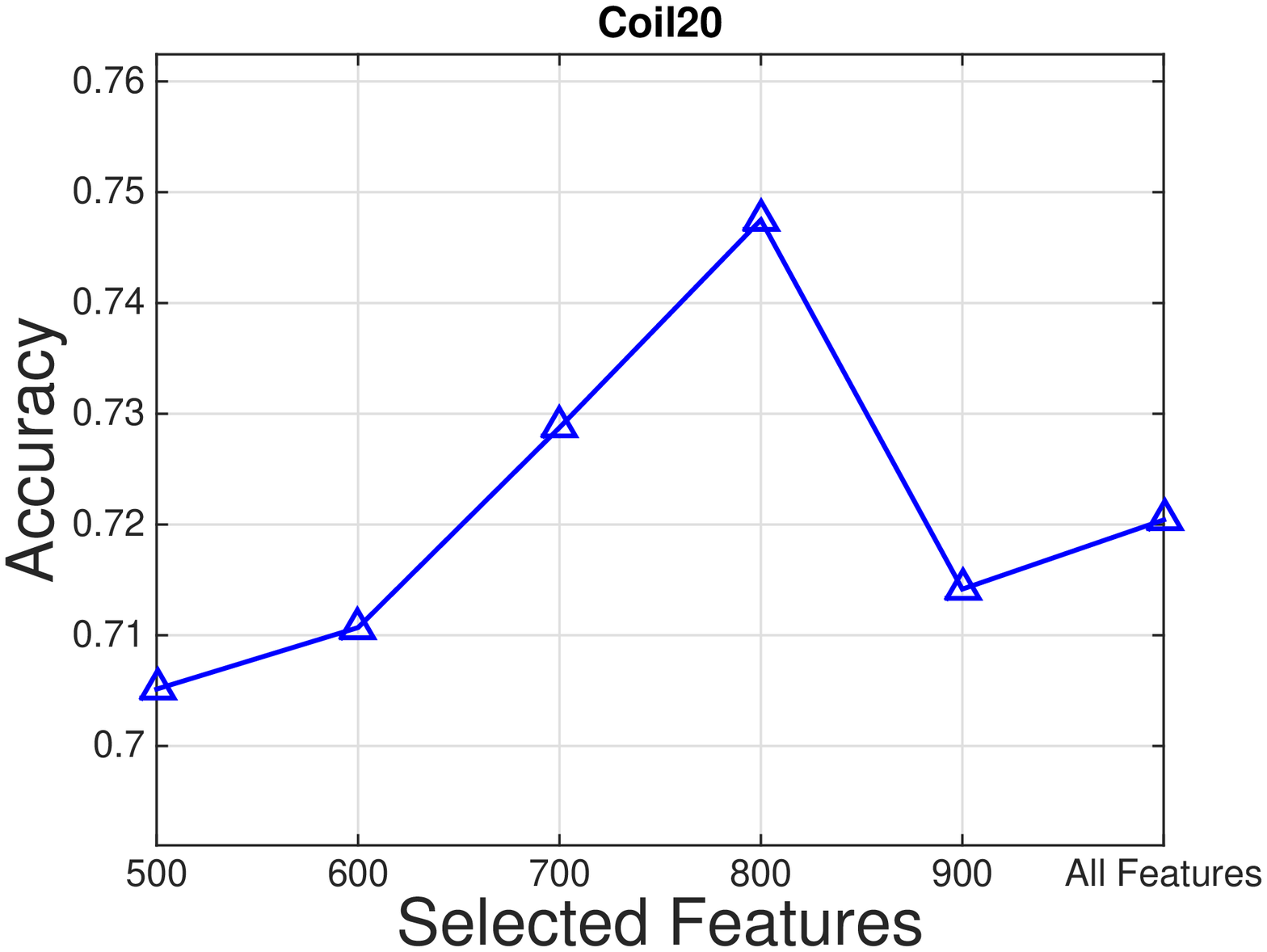}}\vspace{-0mm}\hspace{0mm}
   \subfigure[MNIST]{
   \includegraphics[width=.30\linewidth]
   {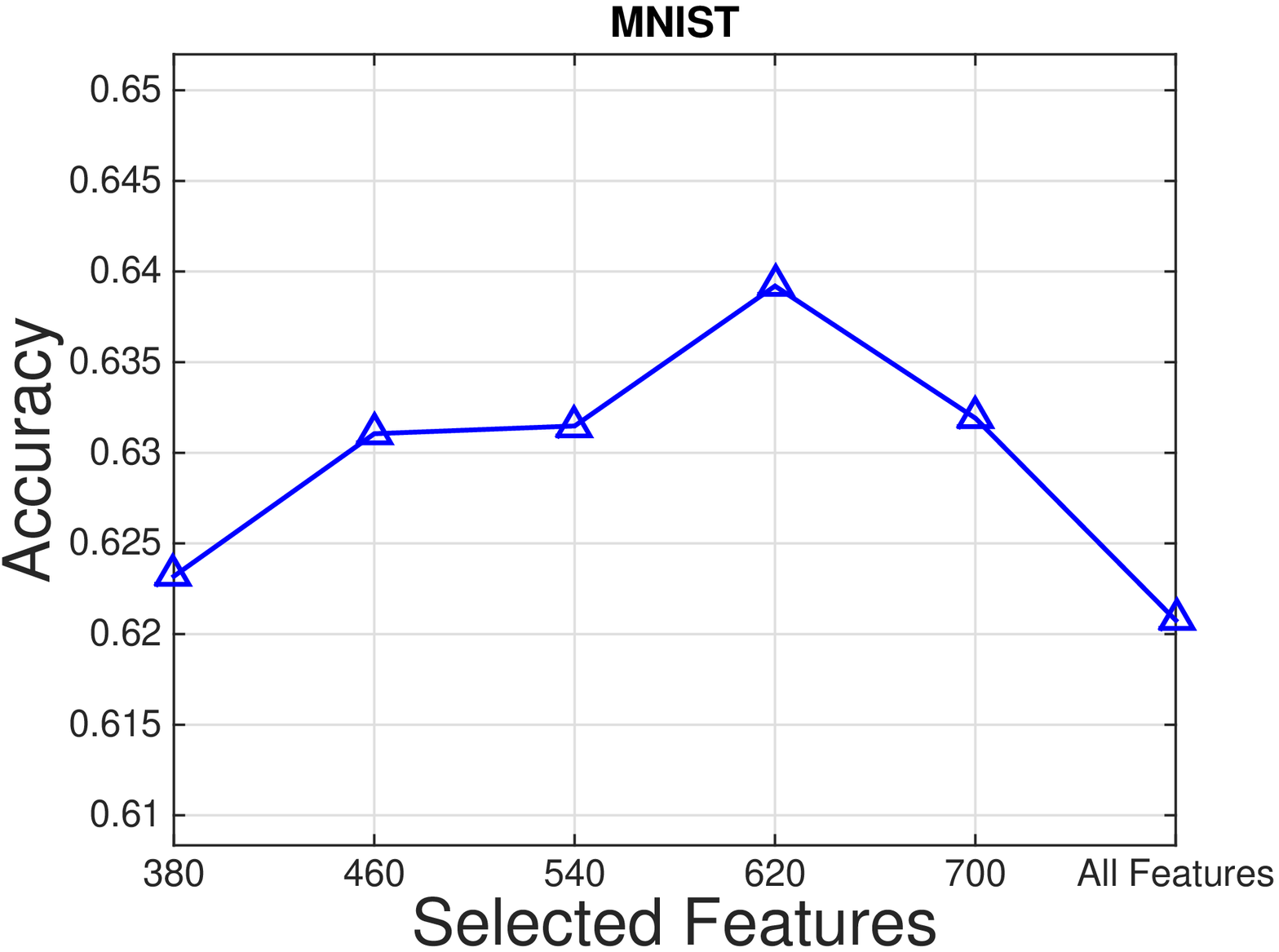}}\vspace{-0mm}\hspace{0mm}
  \subfigure[USPS]{
    \includegraphics[width = .30\linewidth]
    {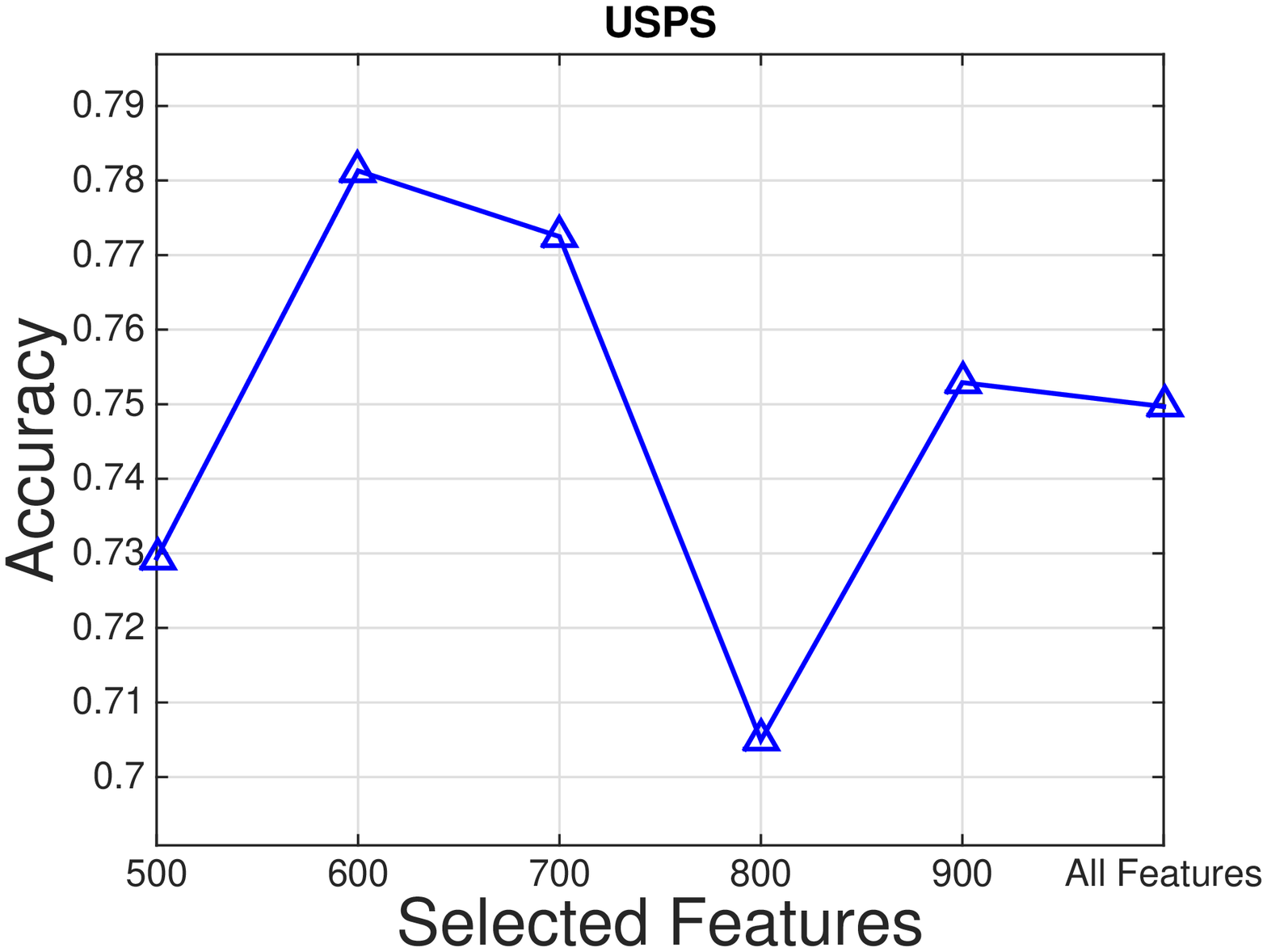}}\vspace{-0mm}\hspace{0mm}
\caption{Performance variation results with respect to the number of selected features using the proposed algorithm over three data sets, COIL20, MNIST, and USPS.}
\label{fig_dim} 
\end{figure}

We study how the number of selected features can affect the performance by conducting an experiment whose results are shown in Figure \ref{fig_dim}. From the figure, performance variations with respect to the number of selected features using the proposed algorithm over three data sets, including COIL20, MNIST, and USPS, have been illustrated. We only adopt \textit{ACC} as the metric. Some observations can be obtained: 1) When the number of selected features is small, e.g. 500 on each data set, the accuracy value is relatively small. 2) With the increase of selected features, performance can peak at a certain point. For example, the performance of our algorithm peaks at 0.7475 on COIL20 when the number of selected features increases to 800. Similarly, 0.6392 (800 selected features) and 0.7813 (600 selected features) are observed on MNIST and USPS, respectively. 3) When all features are in use, the performance is worse than the best. Similar trends can be also observed on the other data sets. It is concluded that our algorithm is able to select distinctive features.

To demonstrate exploiting feature correlation is beneficial to the performance, we conduct an experiment in which parameters $\alpha$ and $p$ are both fixed at 1. $\beta$ varies in a range of $[ 0, 10^{-3}, 10^{-2}, 10^{-1}, 1, 10^1, 10^2, 10^3 ]$. The performance variation results with respect to different $\beta$s are plotted in Figure \ref{fig_2p}. The experiment is conducted over three data sets, i.e. COIL20, MNIST, and USPS. From the results, we can observe that the performance is relatively low, when there is no correlation exploiting in the framework, i.e. $\beta = 0$. The performance always peaks at a certain point when a proper degree of sparsity is imposed to the regularization term. For example, the performance is only 0.6993 when $\beta = 0$ on COIL20. The performance increases to 0.7285 when $\beta=10^1$. Similar observations are also obtained on the other data sets. We can conclude that sparse structure learning on feature coefficient matrix contributes to the performance of our unsupervised feature selection method.
\begin{figure}[t]
  \centering
  \subfigure[COIL20]{
    \includegraphics[width = .30\linewidth]
    {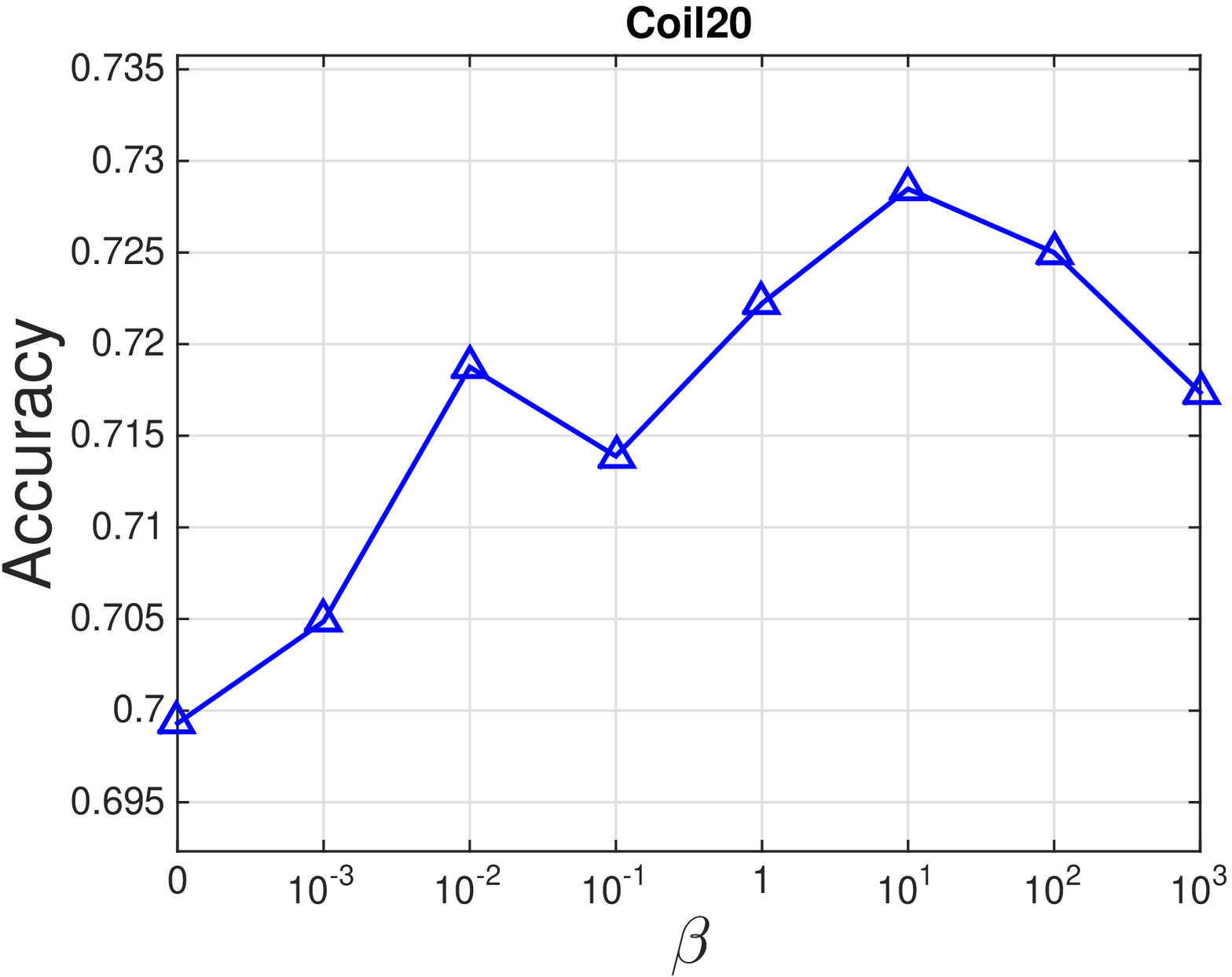}}\vspace{-0mm}\hspace{0mm}
   \subfigure[MNIST]{
   \includegraphics[width=.30\linewidth]
   {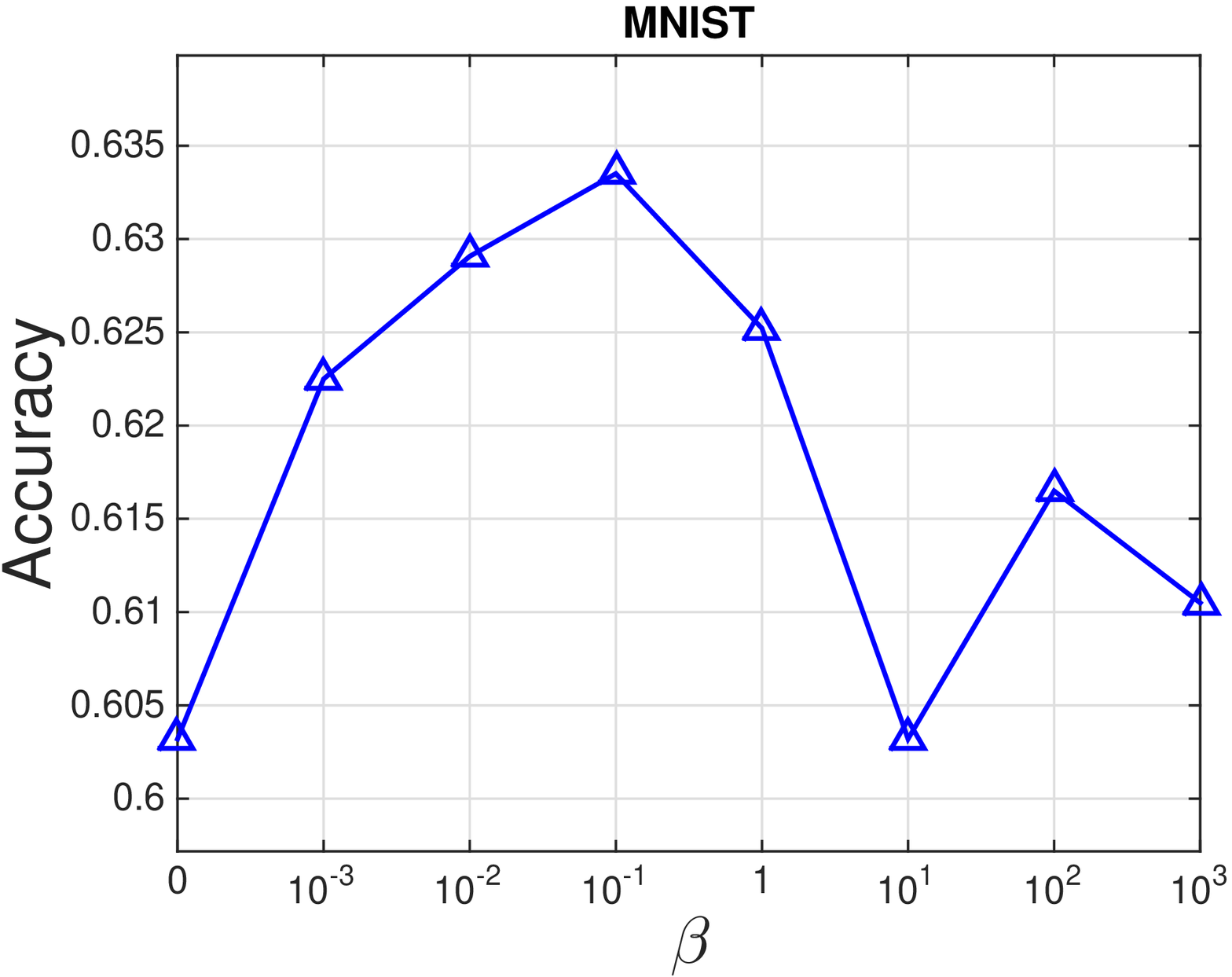}}\vspace{-0mm}\hspace{0mm}
  \subfigure[USPS]{
    \includegraphics[width = .30\linewidth]
    {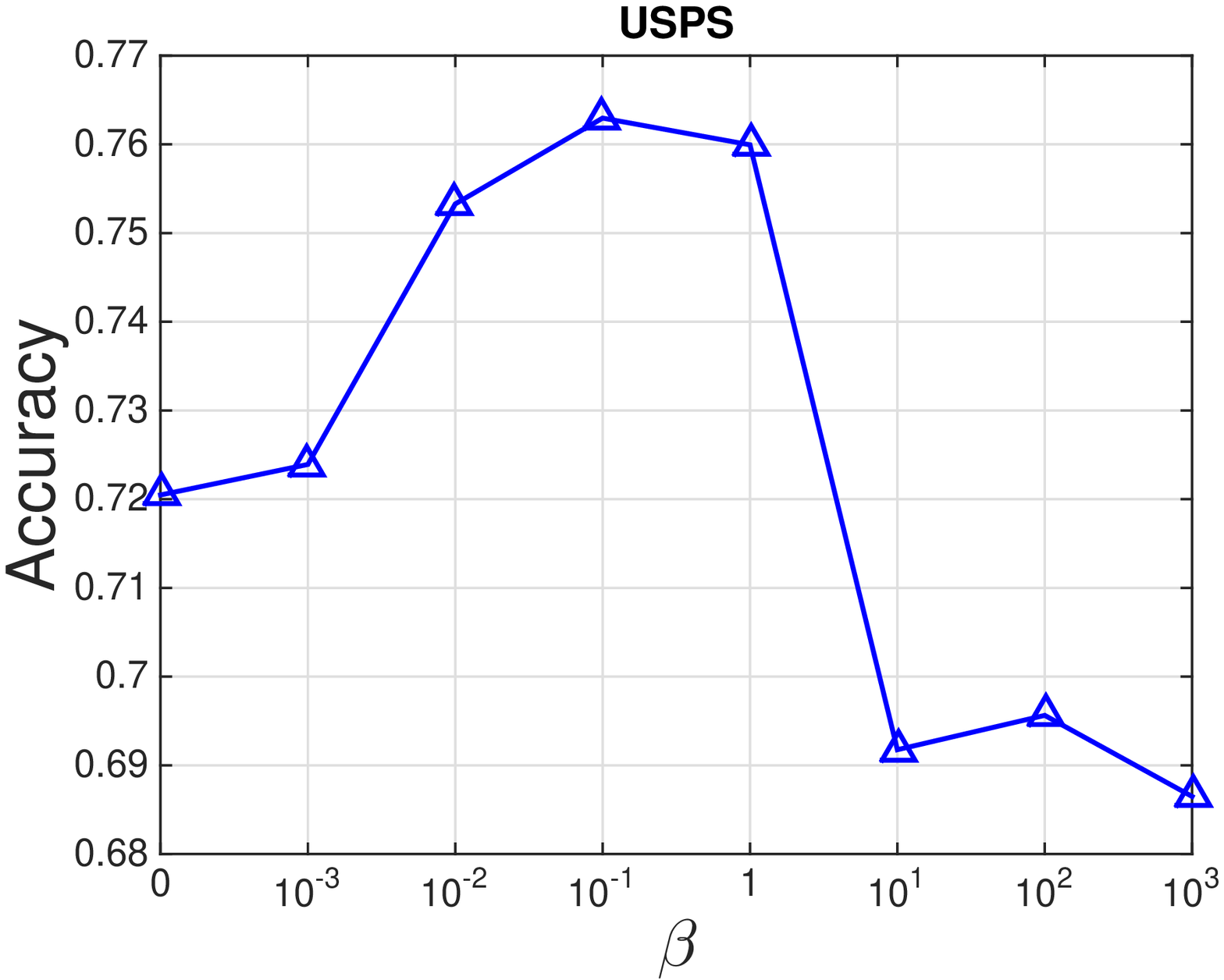}}\vspace{-0mm}\hspace{0mm}
\caption{Performance variation results with respect to different values of regularization parameter, $\beta$s, over three data sets, COIL20, MNIST, and USPS.}
\label{fig_2p} 
\end{figure}
\begin{figure*}[!bt]
\centering
\includegraphics[width=1.07\textwidth]{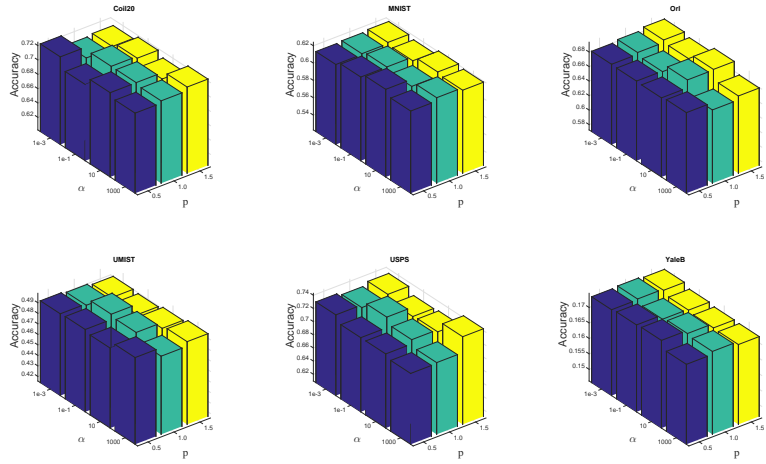}
\caption {Performance variation results under different combinations of $\alpha$s and $p$s. $\beta$ is fixed at $10^{-1}$.}
\label{exp2_1}
\end{figure*}
\subsection{Studies on Parameter Sensitivity and Convergence}
There are three parameters in our algorithms, which are denoted as $\alpha$, $\beta$ and $p$ in \eqref{eq_obj_final}. $\alpha$ and $\beta$ are two regularization parameters while $p$ controls the degree of sparsity. To investigate the sensitivity of the parameters, we conduct an experiment to study how they exert influences on performance. Firstly, we fix $\beta = 10^{-1}$ and derive the performance variations under different combinations of $\alpha$s and $p$s in Figure \ref{exp2_1}. Secondly, $\alpha$ is fixed at $10^{-1}$. The performance variation results with respect to different $\beta$s and $p$s are shown in Figure \ref{exp2_2}. Both $\alpha$ and $\beta$ vary in a range of $[10^{-3}, 10^{-1}, 10^{1}, 10^{1}]$. While $p$ changes in $[0.5, 1.0, 1.5]$. We only take \textit{ACC} as the metric.

To validate that our algorithm will monotonically increase the objective function value in \eqref{eq_obj_final}, we conduct an experiment to demonstrate this fact. In this experiment, all parameters ($\alpha, \beta$, and $p$) in \eqref{eq_obj_final} are fixed at 1. The objective function values and corresponding iteration numbers are drawn in Figure \ref{fig_obj}. We take COIL20, MNIST, and USPS as examples. Similar observations can be also obtained on the other data sets. From the figure, it can be seen that our algorithm converges to the optimum, usually within eight iteration steps, over three data sets. We can then conclude that the proposed method is efficient and effective.

\begin{figure*}[!t]
\centering
\includegraphics[width=1.07\textwidth]{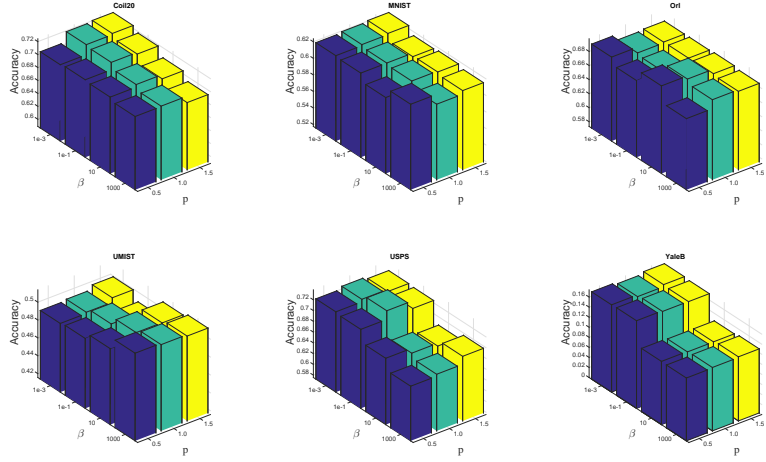}
\caption {Performance variation results under different combinations of $\beta$s and $p$s. $\alpha$ is fixed at $10^{-1}$.}
\label{exp2_2}
\end{figure*}
\begin{figure}
  \centering
  \subfigure[COIL20]{
    \includegraphics[width = .30\linewidth]
    {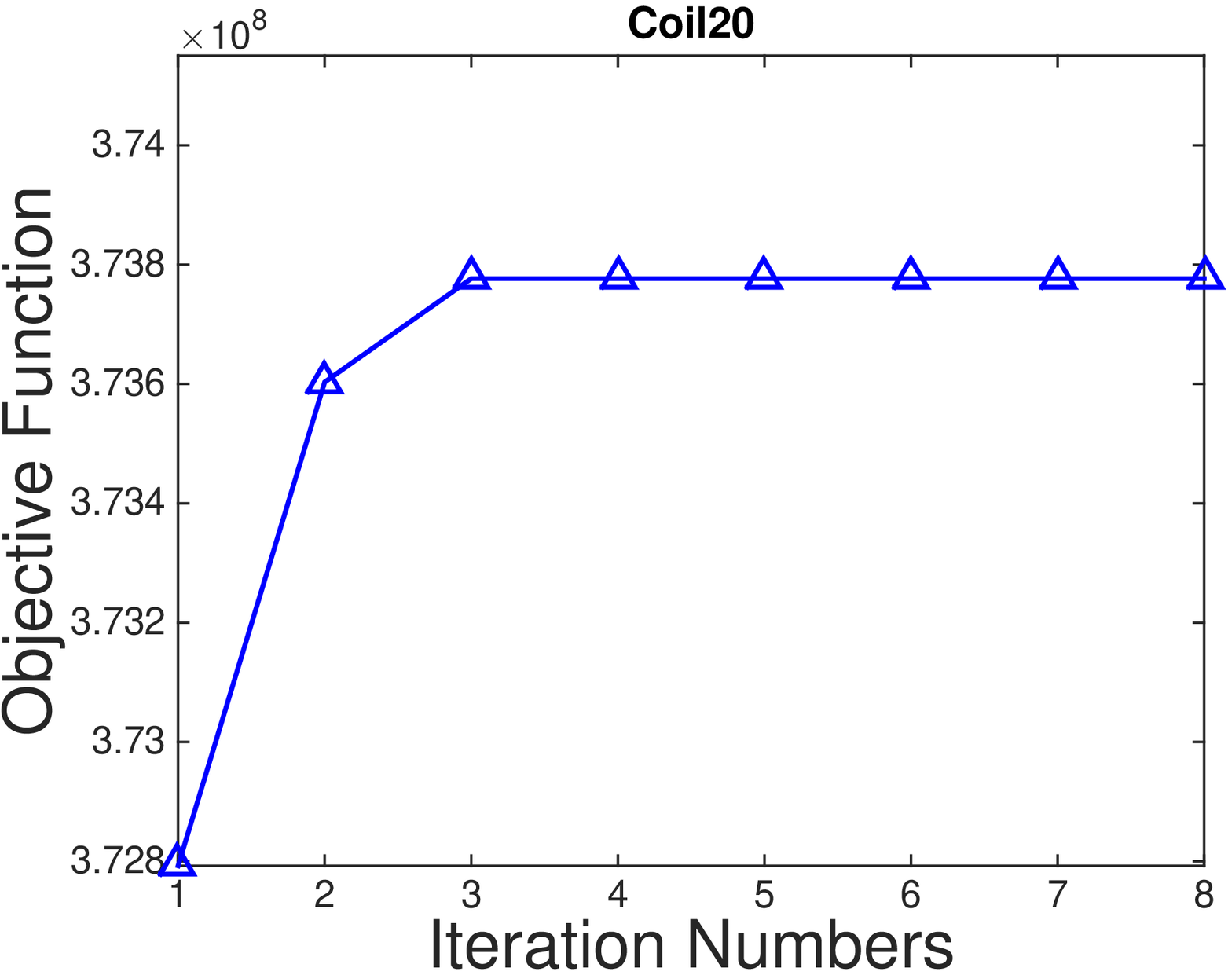}}\vspace{-0mm}\hspace{0mm}
   \subfigure[MNIST]{
   \includegraphics[width=.30\linewidth]
   {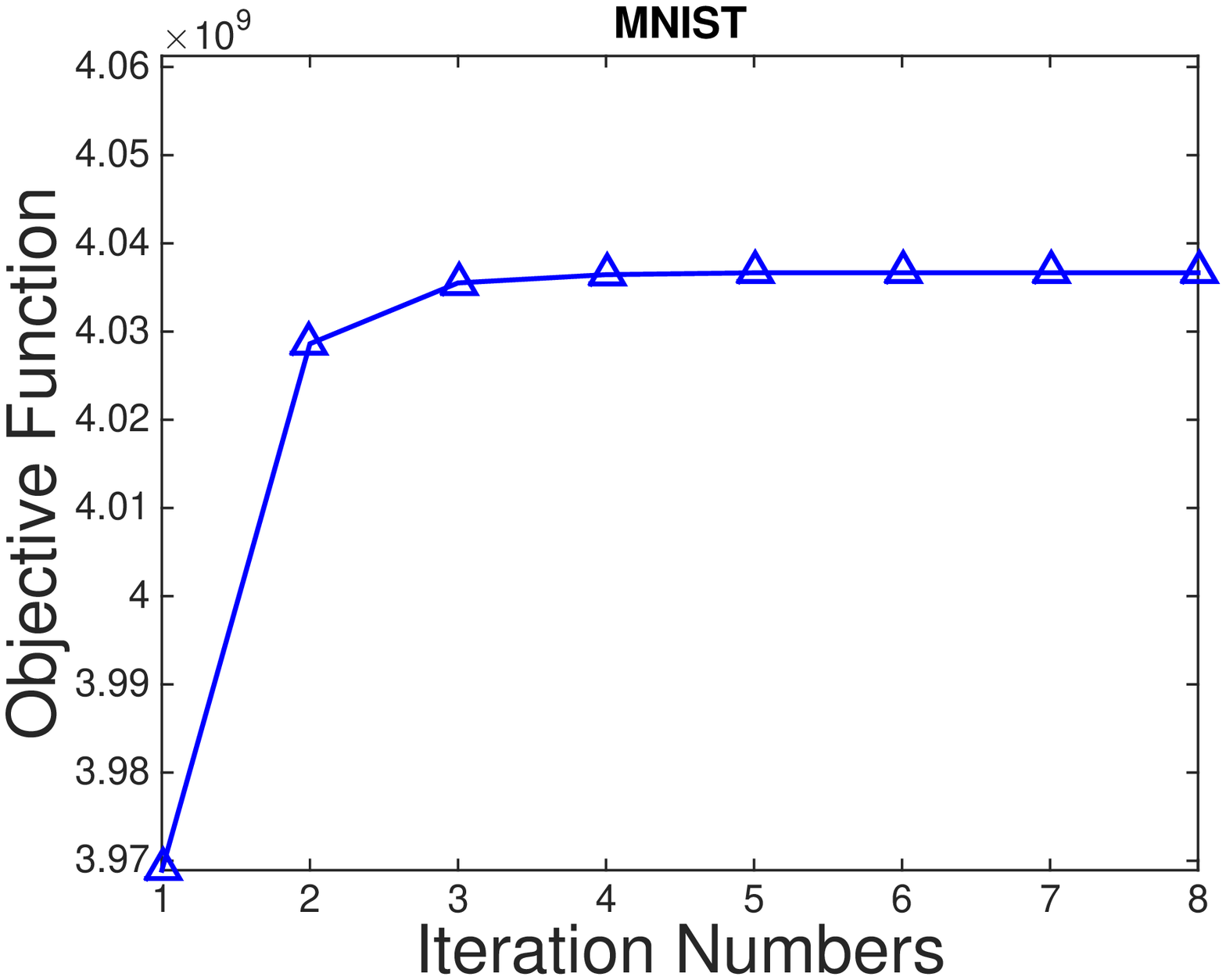}}\vspace{-0mm}\hspace{0mm}
  \subfigure[USPS]{
    \includegraphics[width = .30\linewidth]
    {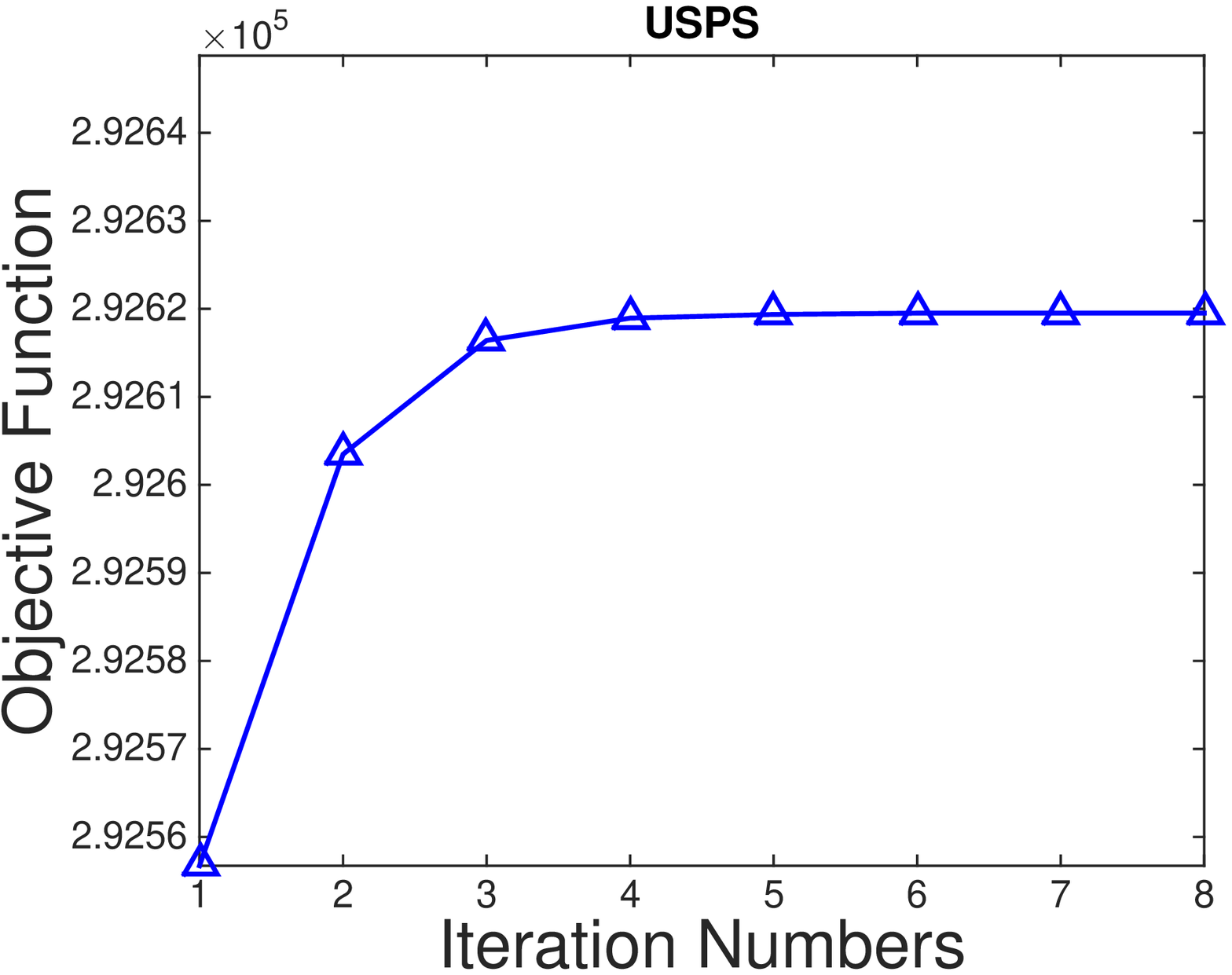}}\vspace{-0mm}\hspace{0mm}
\caption{Objective function values of our proposed objective function in \eqref{eq_obj_final} over three data sets, COIL20, MNIST, and USPS.}
\label{fig_obj} 
\vspace{-2pt}
\end{figure}
\section{Conclusion}
In this paper, an unsupervised feature selection approach has been proposed by using the Maximum Margin Criterion and the sparsity-based model. More specifically, the proposed method seeks to maximize the total scatter on one hand. On the other hand, the within-class scatter is simultaneously considered to minimize. Since there is no label information in an unsupervised scenario, K-means clustering is embedded into the framework jointly. Advantages can be summarized as twofold: First, pseudo labels generated by K-means clustering is beneficial to maximizing class margins in each iteration step. Second, pseudo labels can guide the sparsity-based model to exploit sparse structures of the feature coefficient matrix. Noisy and uncorrelated features can be therefore removed. Since the objective function is non-convex for all variables, we have proposed an algorithm with a guaranteed convergence property. To avoid to rapidly converge to a local optimum which is caused by K-means, we have applied an updating strategy to alleviate the problem. In this way, our proposed method might converge to the global optimum. Extensive experimental results have shown that our method has superior performance against all other compared approaches over six benchmark data sets.
\bibliographystyle{splncs03}
\bibliography{ecml15}

\begin{thebibliography}{10}
\providecommand{\url}[1]{\texttt{#1}}
\providecommand{\urlprefix}{URL }

\bibitem{cai2010unsupervised}
Cai, D., Zhang, C., He, X.: Unsupervised feature selection for multi-cluster
  data. In: ACM International Conference on Knowledge Discovery and Data Mining
  (SIGKDD). pp. 333--342. ACM (2010)

\bibitem{DBLP:journals/tnn/CHANGXJ15}
Chang, X., Nie, F., Wang, S., Yang, Y.: Compound rank-k projections for
  bilinear analysis. {IEEE} Trans. Neural Netw. Learning Syst.  (2015)

\bibitem{chang2014convex}
Chang, X., Nie, F., Yang, Y., Huang, H.: A convex formulation for
  semi-supervised multi-label feature selection. In: AAAI Conference on
  Artificial Intelligence (AAAI) (2014)

\bibitem{chang2014semi}
Chang, X., Shen, H., Wang, S., Liu, J., Li, X.: Semi-supervised feature
  analysis for multimedia annotation by mining label correlation. In: Advances
  in Knowledge Discovery and Data Mining, pp. 74--85. Springer (2014)

\bibitem{chang2015ijcai}
Chang, X., Yang, Y., Hauptmann, A.G., Xing, E.P., Yu, Y.: Semantic concept
  discovery for large-scale zero-shot event detection. In: IJCAI (2015)

\bibitem{du2014multiple}
Du, X., Yan, Y., Pan, P., Long, G., Zhao, L.: Multiple graph unsupervised
  feature selection. Signal Processing  (2014)

\bibitem{duda2012pattern}
Duda, R.O., Hart, P.E., Stork, D.G.: Pattern classification. John Wiley \& Sons
  (2012)

\bibitem{georghiades2001few}
Georghiades, A.S., Belhumeur, P.N., Kriegman, D.: From few to many:
  Illumination cone models for face recognition under variable lighting and
  pose. IEEE Transactions on Pattern Analysis and Machine Intelligence (TPAMI)
  23(6),  643--660 (2001)

\bibitem{han2014semisupervised}
Han, Y., Yang, Y., Yan, Y., Ma, Z., Sebe, N., Zhou, X.: Semisupervised feature
  selection via spline regression for video semantic recognition  26(2),
  252--264 (2015)

\bibitem{he2005laplacian}
He, X., Cai, D., Niyogi, P.: Laplacian score for feature selection. In:
  Advances in Neural Information Processing Systems (NIPS). pp. 507--514 (2005)

\bibitem{HouNLYW14}
Hou, C., Nie, F., Li, X., Yi, D., Wu, Y.: Joint embedding learning and sparse
  regression: {A} framework for unsupervised feature selection. {IEEE} T.
  Cybernetics  44(6),  793--804 (2014)

\bibitem{hull1994adatabase}
Hull, J.J.: A database for handwritten text recognition research. IEEE
  Transactions on Pattern Analysis and Machine Intelligence (TPAMI)  16(5),
  550--554 (may 1994)

\bibitem{kira1992practical}
Kira, K., Rendell, L.A.: A practical approach to feature selection. In:
  International Workshop on Machine Learning. pp. 249--256 (1992)

\bibitem{kononenko1994estimating}
Kononenko, I.: Estimating attributes: analysis and extensions of relief. In:
  Machine Learning: ECML-94. pp. 171--182. Springer (1994)

\bibitem{lecun1998gradient}
LeCun, Y., Bottou, L., Bengio, Y., Haffner, P.: Gradient-based learning applied
  to document recognition. Proceedings of the IEEE  86(11),  2278--2324 (1998)

\bibitem{li2012unsupervised}
Li, Z., Yang, Y., Liu, J., Zhou, X., Lu, H.: Unsupervised feature selection
  using nonnegative spectral analysis. In: AAAI Conference on Artificial
  Intelligence (AAAI). pp. 1026--1032 (2012)

\bibitem{maugis2009variable}
Maugis, C., Celeux, G., Martin-Magniette, M.L.: Variable selection for
  clustering with gaussian mixture models. Biometrics  65(3),  701--709 (2009)

\bibitem{nene1996columbia}
Nene, S.A., Nayar, S.K., Murase, H., et~al.: Columbia object image library
  (coil-20). Tech. rep., Technical Report CUCS-005-96 (1996)

\bibitem{NieHCD10}
Nie, F., Huang, H., Cai, X., Ding, C.H.Q.: Efficient and robust feature
  selection via joint l2, 1-norms minimization. In: NIPS (2010)

\bibitem{qian2013robust}
Qian, M., Zhai, C.: Robust unsupervised feature selection. In: International
  Joint Conference on Artificial Intelligence (IJCAI). pp. 1621--1627. AAAI
  Press (2013)

\bibitem{raileanu2004theoretical}
Raileanu, L.E., Stoffel, K.: Theoretical comparison between the gini index and
  information gain criteria. Annals of Mathematics and Artificial Intelligence
  41(1),  77--93 (2004)

\bibitem{samaria1994parameterisation}
Samaria, F.S., Harter, A.C.: Parameterisation of a stochastic model for human
  face identification. In: IEEE Workshop on Applications of Computer Vision.
  pp. 138--142. IEEE (1994)

\bibitem{senchang2014}
Sen~Wang, Xiaojun~Chang, X.L.Q.Z.S.W.C.: Multi-task support vector machines for
  feature selection with shared knowledge discovery. Signal Processing
  (December 2014)

\bibitem{strehl2003cluster}
Strehl, A., Ghosh, J.: Cluster ensembles---a knowledge reuse framework for
  combining multiple partitions. Journal of Machine Learning Research (JMLR)
  3,  583--617 (2003)

\bibitem{sun2014deep}
Sun, Y., Wang, X., Tang, X.: Deep learning face representation from predicting
  10,000 classes. In: IEEE Conference on Computer Vision and Pattern
  Recognition (CVPR). pp. 1891--1898. IEEE (2014)

\bibitem{tabakhi2014unsupervised}
Tabakhi, S., Moradi, P., Akhlaghian, F.: An unsupervised feature selection
  algorithm based on ant colony optimization. Engineering Applications of
  Artificial Intelligence  32,  112--123 (2014)

\bibitem{WangNH14}
Wang, D., Nie, F., Huang, H.: Unsupervised feature selection via unified trace
  ratio formulation and k-means clustering {(TRACK)}. In: ECML/PKDD (2014)

\bibitem{wang2015embedded}
Wang, S., Tang, J., Liu, H.: Embedded unsupervised feature selection. AAAI
  Conference on Artificial Intelligence (AAAI)  (2015)

\bibitem{wu2014data}
Wu, X., Zhu, X., Wu, G.Q., Ding, W.: Data mining with big data. IEEE
  Transactions on Knowledge and Data Engineering (TKDE)  26(1),  97--107 (2014)

\bibitem{xu2010discriminative}
Xu, Z., King, I., Lyu, M.T., Jin, R.: Discriminative semi-supervised feature
  selection via manifold regularization. IEEE Transactions on Neural Networks
  21(7),  1033--1047 (2010)

\bibitem{YangMHS13}
Yang, Y., Ma, Z., Hauptmann, A.G., Sebe, N.: Feature selection for multimedia
  analysis by sharing information among multiple tasks. {IEEE} Transactions on
  Multimedia  15(3),  661--669 (2013)

\bibitem{yang2011l2}
Yang, Y., Shen, H.T., Ma, Z., Huang, Z., Zhou, X.: l2, 1-norm regularized
  discriminative feature selection for unsupervised learning. In: International
  Joint Conference on Artificial Intelligence (IJCAI). vol.~22, p. 1589.
  Citeseer (2011)

\bibitem{YangZWP08}
Yang, Y., Zhuang, Y., Wu, F., Pan, Y.: Harmonizing hierarchical manifolds for
  multimedia document semantics understanding and cross-media retrieval. {IEEE}
  Transactions on Multimedia  10(3),  437--446 (2008)

\bibitem{zhao2007spectral}
Zhao, Z., Liu, H.: Spectral feature selection for supervised and unsupervised
  learning. In: International Conference on Machine Learning. pp. 1151--1157.
  ACM (2007)

\bibitem{ZhuHYSXL13}
Zhu, X., Huang, Z., Yang, Y., Shen, H.T., Xu, C., Luo, J.: Self-taught
  dimensionality reduction on the high-dimensional small-sized data. Pattern
  Recognition  46(1),  215--229 (2013)

\end{thebibliography}
\end{document}